\documentclass[11pt]{article}

\usepackage{arxiv}

\makeatletter
\theoremstyle{arxivplain}
\newmdtheoremenv[style=arxiv@thmbox]{conjecture}[theorem]{Conjecture}
\makeatother
\crefname{conjecture}{Conjecture}{Conjectures}

\makeatletter
\mdfapptodefinestyle{arxiv@thmbox}{nobreak=true}
\mdfapptodefinestyle{arxiv@propbox}{nobreak=true}
\mdfapptodefinestyle{arxiv@lembox}{nobreak=true}
\mdfapptodefinestyle{arxiv@corbox}{nobreak=true}
\mdfapptodefinestyle{arxiv@defbox}{nobreak=true}
\mdfapptodefinestyle{arxiv@rembox}{nobreak=true}
\makeatother
\usepackage{needspace}
\newcommand{\keep}[1][6\baselineskip]{\Needspace{#1}}

\newcommand{\E}{\mathbb{E}}

\title{Random Matrix Theory of Early-Stopped \\ Gradient Flow: A Transient BBP Scenario}
\author{%
  Florentin Coeurdoux\textsuperscript{1}\thanks{Corresponding author: coeurdoux.florentin@gmail.com. Code to reproduce all figures is available at \url{https://github.com/FlorentinCDX/rmt-early-stopping}}
  \quad
  Gr\'egoire Ferr\'e\textsuperscript{1}
  \quad
  Jean-Philippe Bouchaud\textsuperscript{1,2,3}
  \\[0.7em]
  \normalsize
  \textsuperscript{1}Capital Fund Management, Paris, France \\
  \textsuperscript{2}X-CFM Chair of EconophysiX, Ecole polytechnique, Palaiseau, France \\
  \textsuperscript{3}Acad\'emie des Sciences, Paris, France
}

\date{\today}

\runninghead{Transient BBP Transition in Gradient Flow}

\begin{document}
\maketitle

\begin{abstract}
Empirical studies of trained models often report a transient regime in which signal is detectable in a finite gradient descent time window before overfitting dominates. We provide an analytically tractable random-matrix model that reproduces this phenomenon for gradient flow in a linear teacher--student setting. In this framework, learning occurs when an isolated eigenvalue separates from a noisy bulk, before eventually disappearing in the overfitting regime. The key ingredient is anisotropy in the input covariance, which induces fast and slow directions in the learning dynamics. In a two-block covariance model, we derive the full time-dependent bulk spectrum of the symmetrized weight matrix through a $2\times 2$ Dyson equation, and we obtain an explicit outlier condition for a rank-one teacher via a rank-two determinant formula. This yields a transient Baik--Ben Arous--P\'ech\'e (BBP) transition \cite{baik2005phase}: depending on signal strength and covariance anisotropy, the teacher spike may never emerge, emerge and persist, or emerge only during an intermediate time interval before being reabsorbed into the bulk. We map the corresponding phase diagrams and validate the theory against finite-size simulations. Our results provide a minimal solvable mechanism for early stopping as a transient spectral effect driven by anisotropy and noise.
\end{abstract}

\section{Introduction}
\label{sec:intro}

Understanding when gradient-based learning extracts signal from data, and when it instead amplifies noise and overfits, is a central problem in modern learning theory. A recurring empirical observation is that the spectrum of trained weight matrices often passes through a transient regime in which an isolated eigenvalue separates from a noisy bulk before the spectrum becomes more complicated at later stages of training \cite{martin2021implicit,thamm2022random,staats2023boundary,yunis2024approaching}. Although this phenomenon is frequently associated with the practical success of early stopping, its underlying mechanism remains poorly understood.

In this paper we isolate a minimal, analytically tractable mechanism for such a transient spectral separation. We study gradient flow in a linear teacher--student model and show that anisotropy in the input covariance can produce a \emph{transient BBP transition}: an outlier may emerge from the bulk at intermediate times (with a non-zero overlap with the teacher direction) and later be reabsorbed as training continues, in which case the teacher direction cannot be reconstructed \cite{baik2005phase}. The effect is driven by the coexistence of fast and slow directions in the data covariance. Fast directions learn the teacher early, whereas slow directions continue to inject noise and can eventually hide the signal again.

Our model is deliberately simple. It is not intended to reproduce the full heavy-tailed phenomenology observed in deep nonlinear networks at convergence; rather, it is designed to capture and explain the \emph{bulk-plus-spike} regime in a setting where the full dynamics can be analyzed explicitly. The key solvable case is a two-block covariance model in which the singular values of the input matrix take only two values. In that simplified yet realistic setting, the bulk spectrum of the symmetrized weight matrix is described by a $2\times 2$ Dyson equation, while the rank-one teacher induces a finite-rank perturbation whose outlier condition can be reduced to a rank-two determinant.

Three regimes emerge from our model, depending on signal strength and input covariance anisotropy: (i) a weak-signal regime in which no outlier ever emerges, so the teacher is never spectrally detectable; (ii) a strong-signal regime in which an outlier emerges and remains separated from the bulk at all times; and (iii) an ``early-stopping'' regime in which an outlier exists only during a finite time window, providing a precise random-matrix formulation of early stopping.

\paragraph{Contributions}

In the solvable setting above, we (a) derive a $2\times2$ Dyson equation characterizing the full time-dependent bulk spectrum of the symmetrized weight matrix; (b) establish an explicit rank-two BBP-type outlier condition valid at every finite time $t$; (c) classify the possible learning regimes and give closed-form phase diagrams in the $(\theta,t)$ and $(\theta,\lambda_-)$ planes; (d) identify numerically an optimal early-stopping time, defined through the teacher--eigenvector overlap, that lies strictly inside the transient BBP window; and (e) map the regimes as a function of the sample ratio $\gamma$, and show numerically that the same transient persists for a continuous power-law spectrum and in nonlinear networks (Sections~\ref{subsec:phase_diagrams}--\ref{subsec:beyond}).

\paragraph{Related works}
The BBP phase transition, first discussed by Edwards and Jones~\cite{edwards1976eigenvalue} and established by Baik, Ben Arous and P\'ech\'e~\cite{baik2005phase}, has since become a standard tool in spiked random matrix models and finite-rank perturbation theory~\cite{benaych2011eigenvalues}. Exact learning dynamics in linear and deep-linear models, including gradient-flow learning and early stopping, have been studied extensively~\cite{saxe2014exact,advani2020highdim,ali2019continuous,lampinen2019analytic}, also with anisotropic inputs~\cite{watanabe2026anisotropic}; spectral bias is well understood in kernel regression~\cite{canatar2021spectral}, and the exact high-dimensional dynamics of SGD under anisotropic data has been derived~\cite{paquette2025homogenization,collinswoodfin2024hitting}. Spectral properties of weight matrices have been linked to generalisation empirically~\cite{martin2021implicit,thamm2022random,staats2023boundary,yunis2024approaching,olsen2025sgd,aarts2025dyson} and, in two-layer networks, theoretically: spikes that emerge in the weights after large gradient steps carry the learned features~\cite{ba2022high,moniri2024theory,dandi2025random}, and a teacher with directions of decreasing importance is learned through a cascade of spectral transitions~\cite{defilippis2026scaling,defilippis2026optimal}. Dynamical BBP transitions appear in the Hessian of two-layer networks trained by online SGD on Gaussian mixtures~\cite{benarous2025spectral,benarous2026local}, an inverse BBP transition in the dynamical Hessian of gradient descent for phase retrieval~\cite{bonnaire2025role}, and a related transition in the Hessian spectrum during feature learning~\cite{montanari2026phase}. More broadly, random matrix methods have become increasingly useful in the study of high-dimensional learning dynamics and implicit regularisation \cite{barbier2019optimal,ajanki2017universality,ajanki2019,maillard2019high,atanasov2024scaling,bordelon2026disordered,liao2025rmt}. Our contribution is to provide an arguably oversimplified but explicit random-matrix mechanism for a \emph{transient} spectral spike during learning, i.e. a BBP transition followed by an inverse BBP transition, and to tie this transient spectral geometry to a computable optimal stopping time.

\section{Model and Setup}
\label{sec:model}

\subsection{Gradient flow on a linear model}
\label{sec:linearmodel}

We consider inputs $X\in\mathbb R^{N\times M}$ and targets $Y\in\mathbb R^{N\times M}$ in the high-dimensional regime
\[
N,M\to\infty,
\qquad
\frac{M}{N}\to\gamma\geq 1.
\]

We train a linear model with weight matrix $A\in\mathbb R^{N\times N}$ by minimizing
\begin{equation}
\label{eq:Astar}
A_{\mathrm{LS}}=\underset{A}{\arg\min}\,J(A),
\qquad
J(A)=\frac12\|AX-Y\|_F^2.
\end{equation}

Although the minimizer is explicit and reads $A_{\mathrm{LS}} \equiv YX^\top(XX^\top)^{-1}$, our interest is not the optimizer itself but the spectral evolution of the gradient-flow trajectory
\[
\frac{dA_t}{dt}=-\nabla J(A_t),
\qquad
A_0=A_{\mathrm{init}}.
\]

Since
\[
\nabla J(A)=(AX-Y)X^\top = AXX^\top-YX^\top,
\]
the flow solves the linear matrix ODE
\[
\dot A_t=-A_tXX^\top+YX^\top.
\]
Its explicit solution is
\begin{equation}
\label{eq:At}
A_t
=
A_{\mathrm{init}}e^{-tXX^\top}
+
YX^\top(XX^\top)^{-1}\bigl(I-e^{-tXX^\top}\bigr).
\end{equation}

This is a standard semigroup evolution, and in the present least-squares setting it provides a natural continuous-time framework for early stopping \cite{ali2019continuous}. At long times it converges to the least-squares estimator $A_\infty \equiv A_{\mathrm{LS}}$. The key point for us is that the learning speed is anisotropic: the eigenvalues of $XX^\top$ determine how fast the corresponding directions are learned or damped. This observation motivates the two-block covariance model introduced below.

\subsection{Random matrix assumptions}

To model a teacher--student setting, we assume that the targets are generated by a rank-one teacher
\begin{equation}
\label{eq:signal}
Y = A_{\mathrm{teach}}^\theta X + Z,
\qquad
A_{\mathrm{teach}}^\theta =\theta\,vv^\top,
\end{equation}
where $v\in\mathbb R^N$ is a fixed unit vector and $\theta\ge 0$ is the signal amplitude. The noise matrix $Z\in\mathbb R^{N\times M}$ has i.i.d.\ $\mathcal N(0,1/M)$ entries, independent of everything else. We take $A_{\mathrm{init}}$ to be a GOE matrix at the usual $1/\sqrt{N}$ scale, independent of $X$ and $Z$.

Substituting \eqref{eq:signal} into \eqref{eq:At} yields the decomposition
\begin{equation}
\label{eq:At_decomposed}
A_t^\theta
=
\underbrace{A_{\mathrm{init}}e^{-tXX^\top}
+
ZX^\top(XX^\top)^{-1}\bigl(I-e^{-tXX^\top}\bigr)}_{\text{noise part}}
+
\underbrace{\theta\,vv^\top\bigl(I-e^{-tXX^\top}\bigr)}_{\text{teacher part}}.
\end{equation}

We consider the additive symmetrization
\[
S_t^\theta := A_t^\theta + (A_t^\theta)^\top,
\]
because it is the simplest construction that preserves the bulk-plus-spike structure of the problem while remaining fully solvable at finite time: the noise contribution becomes a Wigner-type matrix with a piecewise-constant variance profile, and the teacher contribution becomes a finite-rank symmetric perturbation. The more standard multiplicative choice $A_t A_t^\top$ would introduce a Wishart-type bulk whose time-dependent spectrum requires a substantially heavier analysis; we return to this in the conclusion.
Using \eqref{eq:At_decomposed}, we obtain
\begin{equation}
\label{eq:St_split}
S_t^\theta
=
S_t^0
+
\theta\Bigl[
vv^\top\bigl(I-e^{-tXX^\top}\bigr)
+
\bigl(I-e^{-tXX^\top}\bigr)vv^\top
\Bigr],
\end{equation}
where the pure-noise part is
\begin{equation}
\label{eq:St_noise}
S_t^0
=
A_{\mathrm{init}}e^{-tXX^\top}
+
e^{-tXX^\top}A_{\mathrm{init}}^\top
+
ZX^\top(XX^\top)^{-1}\bigl(I-e^{-tXX^\top}\bigr)
+
\bigl(I-e^{-tXX^\top}\bigr)(XX^\top)^{-1}XZ^\top.
\end{equation}

Our goal is to describe the spectrum of $S_t^\theta$ for all times $t\ge 0$ and all $\theta\ge 0$. Note that we alternatively write~$S_t$ instead of~$S_t^0$ for notational simplicity.

\subsection{Two-block structure of the input covariance}
\label{sec:twoblock}

To obtain an analytically tractable model, we assume that the singular values of $X$ take only two values: a fraction $\alpha$ are equal to $\lambda_+=1$ and a fraction $1-\alpha$ are equal to $\lambda_- \in (0,1]$. Equivalently, $XX^\top$ has two eigenvalues, $1$ and $\lambda_-^2$, with proportions $\alpha$ and $1-\alpha$. The smaller $\lambda_-$ is, the more anisotropic the learning problem becomes.\footnote{This is where our assumption that $\gamma \geq 1$ is needed. Indeed, when $\gamma <1$ the matrix $XX^\top$ has a null space of dimension $(1- \gamma)N$, which leads to a $3\times 3$ Dyson system. See Appendix \ref{app:gammalt1}.}

In order to simplify~\eqref{eq:St_noise}, let us write the compact singular value decomposition
\[
X = U\Lambda V^\top,
\]
where $U\in O(N)$, $V\in\mathbb R^{M\times N}$ has orthonormal columns, and
\[
\Lambda = \operatorname{diag}(\underbrace{1,\dots,1}_{\alpha N},\underbrace{\lambda_-,\dots,\lambda_-}_{(1-\alpha)N})\in\mathbb R^{N\times N}.
\]

Conjugating by $U$ and using the rotational invariance of $A_{\mathrm{init}}$ and of the Gaussian noise, we obtain
\[
\widetilde A_{\mathrm{init}}:=U^\top A_{\mathrm{init}}U,\qquad
\widetilde Z:=U^\top ZV,\qquad
\widetilde v:=U^\top v,
\]
with $\widetilde A_{\mathrm{init}}$ distributed as GOE (Gaussian Orthogonal Ensemble) and $\widetilde Z$ having i.i.d.\ Gaussian entries of variance $1/M$. In this basis,
\begin{equation}
\label{eq:conjugated}
U^\top A_t^\theta U
=
\widetilde A_{\mathrm{init}}e^{-t\Lambda^2}
+
\widetilde Z\Lambda^{-1}\bigl(I-e^{-t\Lambda^2}\bigr)
+
\theta\,\widetilde v\widetilde v^\top\bigl(I-e^{-t\Lambda^2}\bigr).
\end{equation}

Therefore the symmetrized matrix has the same spectrum as
\begin{equation}
\label{eq:tilde_St_split}
\widetilde S_t^\theta
=
\widetilde S_t^0
+
\theta\bigl(\widetilde v w_t^\top + w_t \widetilde v^\top\bigr),
\qquad
w_t := \bigl(I-e^{-t\Lambda^2}\bigr)\widetilde v,
\end{equation}
where
\begin{equation}
\label{eq:tilde_St_noise}
\widetilde S_t^0
=
\widetilde A_{\mathrm{init}}e^{-t\Lambda^2}
+
e^{-t\Lambda^2}\widetilde A_{\mathrm{init}}^\top
+
\widetilde Z\Lambda^{-1}\bigl(I-e^{-t\Lambda^2}\bigr)
+
\bigl(I-e^{-t\Lambda^2}\bigr)\Lambda^{-1}\widetilde Z^\top.
\end{equation}

This representation makes the structure transparent: the noise part is a two-block Wigner-type matrix, while the teacher contributes a finite-rank perturbation which is rank one only in the isotropic case $\lambda_-=1$ and rank two otherwise.

\section{Bulk spectrum via the $2\times 2$ Dyson equation}
\label{sec:spectrum}

We now characterize the bulk spectrum of the noise matrix in the abstract two-block setting.

\subsection{Two-block Wigner-type ensemble}

Define the block matrices
\[
D_{1,t}
=
\operatorname{diag}(\underbrace{a_t,\dots,a_t}_{\alpha N},\underbrace{b_t,\dots,b_t}_{(1-\alpha)N}),
\qquad
D_{2,t}
=
\operatorname{diag}(\underbrace{c_t,\dots,c_t}_{\alpha N},\underbrace{d_t,\dots,d_t}_{(1-\alpha)N}),
\]
with (for $\lambda_+=1$)
\begin{equation}
\label{eq:params}
a_t=e^{-t},
\qquad
b_t=e^{-\lambda_-^2 t},
\qquad
c_t=1-e^{-t},
\qquad
d_t=\frac{1-e^{-\lambda_-^2 t}}{\lambda_-}.
\end{equation}

Let $W\in\mathbb R^{N\times N}$ be GOE with the usual $1/N$ scaling, and let $\Gamma\in\mathbb R^{N\times N}$ have i.i.d.\ entries of variance $1/M$, independent of $W$. Then \eqref{eq:tilde_St_noise} can be rewritten as
\begin{equation}
\label{eq:S_def}
S_t = WD_{1,t} + D_{1,t}W + \Gamma D_{2,t} + D_{2,t}\Gamma^\top.
\end{equation}

In the limit $N,M\to\infty$ with $M/N\to\gamma \geq 1$, the off-diagonal variance profile is piecewise constant:
\[
N\,\mathbb E [S_{ij}^2] \to \sigma_{\tau(i),\tau(j)},
\]
where $\tau(i)\in\{A,B\}$ is the block label of $i$ and
\begin{equation}
\label{eq:variance_profile}
\sigma_{AA}(t)=4a_t^2+\frac{2c_t^2}{\gamma},
\qquad
\sigma_{BB}(t)=4b_t^2+\frac{2d_t^2}{\gamma},
\qquad
\sigma_{AB}(t)=(a_t+b_t)^2+\frac{c_t^2+d_t^2}{\gamma}.
\end{equation}
Thus $S$ is a Wigner-type matrix with a two-block variance profile \cite{ajanki2019,ajanki2017universality}.

\subsection{The $2\times 2$ Dyson equation}

Let
\[
m_A(z,t)=\lim_{N\to\infty}\frac{1}{\alpha N}\sum_{i\in A}\bigl[(zI-S_t)^{-1}\bigr]_{ii},
\qquad
m_B(z,t)=\lim_{N\to\infty}\frac{1}{(1-\alpha)N}\sum_{i\in B}\bigl[(zI-S_t)^{-1}\bigr]_{ii}.
\]

In the large-$N$ limit, these partial Stieltjes transforms concentrate to deterministic limits satisfying the following closed self-consistent system.

\begin{proposition}[Dyson equation]
\label{prop:dyson}
For each $z\in\mathbb C^+$ and $t\geq 0$, the pair $(m_A(z,t),m_B(z,t))$ is the unique solution of
\begin{equation}
\label{eq:dyson}
\begin{cases}
\displaystyle \frac{1}{m_A(z,t)} = z - \alpha\,\sigma_{AA}(t)\,m_A(z,t) - (1-\alpha)\,\sigma_{AB}(t)\,m_B(z,t), \\[8pt]
\displaystyle \frac{1}{m_B(z,t)} = z - \alpha\,\sigma_{AB}(t)\,m_A(z,t) - (1-\alpha)\,\sigma_{BB}(t)\,m_B(z,t),
\end{cases}
\end{equation}
with the Stieltjes-branch conditions $\Im m_A(z,t),\Im m_B(z,t)<0$ for $\Im z>0$ and $m_A(z,t),m_B(z,t)\sim 1/z$ as $|z|\to\infty$. The global Stieltjes transform and limiting spectral density are
\begin{equation}
\label{eq:mz}
m(z,t)=\alpha\,m_A(z,t)+(1-\alpha)\,m_B(z,t),
\qquad
\rho_{S_t}(\lambda)=-\frac{1}{\pi}\lim_{\eta\downarrow 0}\Im\,m(\lambda+i\eta,t).
\end{equation}
\end{proposition}

A short proof of this result is given in Appendix~\ref{app:proof_dyson}.
Note that when $\alpha=1$ or $\alpha=0$, \eqref{eq:dyson} reduces to the scalar Dyson equation of a semicircle law. When $a=b$ and $c=d$, one has $\sigma_{AA}=\sigma_{AB}=\sigma_{BB}$ and the two-block system again collapses to a single semicircle. In the general case the complex behaviour of the spectral dynamics is illustrated in Figure~\ref{fig:bulk} below. To conclude this part we mention that, should we have postulated an $n$-block structure, we would have obtained $n$ coupled equations for the corresponding set $\{m_{A_1}(z), m_{A_2}(z), \dots, m_{A_n}(z) \}$. The $\gamma<1$ case, which introduces a frozen null-space block, leads to a $3\times 3$ Dyson system and is treated in Appendix~\ref{app:gammalt1}; the resulting permanent noise narrows but does not eliminate the transient regime discussed below.

\subsection{Empirical validation and dynamical phases}
\label{subsec:bulk_empirical}

We validate the Dyson prediction by simulating the full gradient flow \eqref{eq:At} at  logarithmically spaced times $t\in[0.1,2000]$, averaging over 20 independent realizations with $N=500$, $\gamma=1$, $\alpha=0.5$, $\lambda_+=1$, $\lambda_-=0.1$. At each time we form $S_t=A_t+A_t^\top$ and compare the empirical spectral density with the theoretical density $\rho_{S_t}$ obtained from the $2\times 2$ Dyson equation via the Stieltjes inversion formula \eqref{eq:mz}.

\begin{figure}[ht]
  \centering
  \includegraphics[width=\linewidth]{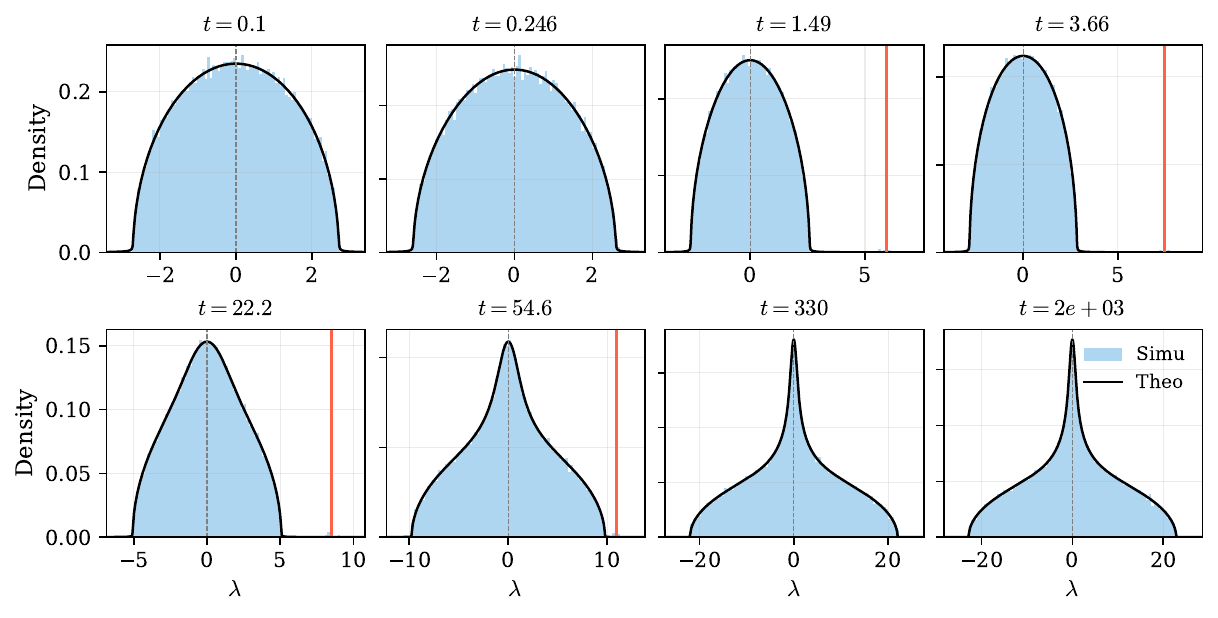}
  \caption{\small Empirical spectral density of $S_t$ (blue histogram, 20 realizations, $N=500$) versus the theoretical density from the $2\times2$ Dyson equation \eqref{eq:dyson}--\eqref{eq:mz} (black curve) at eight training times. Parameters: $\lambda_+=1$, $\lambda_-=0.1$, $\gamma=1$, and $\alpha=0.5$. In the signal case $\theta=6$, the red vertical line marks the outlier location as predicted in Section~\ref{sec:bbp} when an outlier is present. One clearly sees how the outlier first emerges from the bulk before being reabsorbed.}
  \label{fig:bulk}
\end{figure}

Three dynamical phases are visible in Figure~\ref{fig:bulk}. At early times the spectrum is close to a single semicircle, since both blocks still behave similarly and the matrix is dominated by the GOE initialization. At intermediate times the two learning scales separate, the variance profile becomes genuinely two-block, and the bulk deforms accordingly. At very late times the GOE contribution has essentially decayed, while the factor $(XX^\top)^{-1}$ amplifies the slow block and broadens the support. Across the full dynamical range, the theoretical curve obtained from the Dyson equation agrees well with the empirical histogram. When the rank-one teacher is switched on  ($\theta=6$, red vertical lines), one clearly sees how an outlier first detaches from the bulk at intermediate times and is then reabsorbed at late times; the full analysis of this transient BBP transition is the subject of the next section. Numerical experiments with a full-rank power-law input spectrum, where the covariance eigenvalues are drawn from a power-law distribution, confirm that the same three dynamical phases persist well beyond the two-block setting (Figure~\ref{fig:fullrank_overlap} and Appendix~\ref{app:fullrank}).

\keep[10\baselineskip]\section{Rank-two transient BBP transition}
\label{sec:bbp}

\subsection{Outlier condition for a rank-one teacher}

We now analyze the effect of the rank-one teacher on~$S_t$
by considering
\[
S_t^\theta
=
S_t + \theta\bigl(\widetilde v w_t^\top + w_t\widetilde v^\top\bigr),
\qquad
w_t=(I-e^{-t\Lambda^2})\widetilde v.
\]
Because $w_t$ is proportional to $\widetilde v$ only in the isotropic case $\lambda_-=1$, the perturbation is genuinely rank two at finite time, and the standard rank-one BBP formula \cite{benaych2011eigenvalues} does not apply directly. We can however reduce the problem to a $2\times 2$ determinantal condition.

\begin{proposition}[Outlier condition]
\label{prop:outlier}
Let $L_-(t)$ and $L_+(t)$ denote the lower and upper edges of the limiting bulk spectrum of $S_t$. For $\theta\ge 0$, the perturbed matrix $S_t^\theta$ has an outlier eigenvalue $\xi_{\theta,t}$ if and only if it solves
\begin{equation}
\label{eq:outlier}
\bigl(1-\theta\,\psi_t(\xi_{\theta,t})\bigr)^2
=
\theta^2\,\phi_t(\xi_{\theta,t})\chi_t(\xi_{\theta,t}),
\qquad
\xi_{\theta,t}\in\mathbb R\setminus [L_-(t),L_+(t)],
\end{equation}
where the deterministic functions $\phi_t,\psi_t,\chi_t$ are defined in Appendix~\ref{app:BBP}, Eq.~\eqref{eq:quad_form}.
\end{proposition}
Proposition~\ref{prop:outlier} is proved in Appendix~\ref{app:BBP}. \\

\begin{remark}[Isotropic limit]
In the isotropic case $\lambda_-=1$, one has $w_t=f_t\widetilde v$ with $f_t=1-e^{-t}$. Hence $\psi_t(z)=f_t\phi_t(z)$ and $\chi_t(z)=f_t^2\phi_t(z)$, so \eqref{eq:outlier} reduces to
\[
1-2\theta f_t\,\phi_t(\xi_{\theta,t})=0,
\]
which is exactly the standard rank-one BBP equation with effective spike strength $2\theta f_t$.
\end{remark}

\begin{remark}[Boundary times]
At $t=0$, one has $w_t=0$, hence $\psi_0\equiv\chi_0\equiv 0$, and \eqref{eq:outlier} has no solution outside the bulk. At the other extreme, as $t\to\infty$, one has $w_t\to \widetilde v$, so $\phi_t(z),\psi_t(z),\chi_t(z)\to m_\infty(z)$ and \eqref{eq:outlier} reduces to
\[
1-2\theta\,m_\infty(\xi_{\theta,\infty})=0,
\]
which is the usual steady-state spike equation for the limiting matrix $S_\infty^\theta$.
\end{remark}

\subsection{BBP thresholds and signal regimes}

The upper outlier first appears when \eqref{eq:outlier} acquires a solution at the upper edge $L_+(t)$. The corresponding critical signal level is
\begin{equation}
\label{eq:theta_c}
\theta_c(t)
:=
\frac{1}{\psi_t(L_+(t)^+) + \sqrt{\phi_t(L_+(t)^+)\chi_t(L_+(t)^+)}}.
\end{equation}

For fixed signal strength $\theta$, define the fitting set
\[
\mathcal T_\theta := \{t>0:\ \theta>\theta_c(t)\}.
\]
From our observations, we can state the following conjecture. 
\begin{conjecture}[Three signal regimes]
\label{thm:three_phases}
In the limit $N,M\to\infty$ with $M/N\to\gamma>0$, exactly one of the following scenarios occurs:
\begin{enumerate}
  \item \textbf{Weak-signal regime:} $\mathcal T_\theta=\varnothing$. No outlier ever separates from the bulk, and the teacher direction is never spectrally detectable.
  \item \textbf{Strong-signal regime:} $\mathcal T_\theta=(t_1,\infty)$ for some $t_1>0$. An outlier emerges at time $t_1$ and remains separated from the bulk thereafter.
  \item \textbf{Early-stopping regime:} $\mathcal T_\theta=(t_1,t_2)$ for some $0<t_1<t_2<\infty$. An outlier exists only on a finite time window. This is the transient BBP regime associated with early stopping.
\end{enumerate}
\end{conjecture}

In the early-stopping regime, $t_1$ and $t_2$ are interpreted as follows: for $t<t_1$ the model is underfitted and the teacher direction cannot be determined; for $t_1<t<t_2$ the teacher is spectrally detectable in the sense that the eigenvector $u^\star$ corresponding to the spike has a computable, non-zero overlap $(u^\star \cdot v)^2$ with the teacher direction (see e.g. \cite{potters2020first}, section 14.2.2); and for $t>t_2$ the outlier is reabsorbed into the bulk, corresponding to a late-time overfitting phase. The optimal stopping time, defined as the $t_{\rm opt} \in (t_1,t_2)$ that maximizes $(u^\star \cdot v)^2$, is discussed in Section~\ref{subsec:overlap} below.

\par\noindent
  \begin{minipage}[c]{0.5\textwidth}
    We illustrate this behaviour by plotting the evolution in time of $\theta_c(t)$ in Figure~\ref{fig:theta_c}. When $\lambda_+=\lambda_-=1$, the function is monotone, so that only regimes 1 and 2 are possible depending on the value of $\theta$. On the contrary, in the anisotropic regime $\lambda_-=0.1$, the function is no longer monotone: $\theta_c$ first decreases (the fast block reveals the teacher) and then increases again (the slow block, amplified by $\Lambda^{-1}$, injects enough noise to hide the signal). This means that, for an intermediate range of values of $\theta$, a spike is visible only during a finite time interval -- this is the early-stopping regime, and the minimum of $\theta_c$ is the best detection threshold achievable along the whole training trajectory.
  \end{minipage}%
  \hfill
  \begin{minipage}[c]{0.45\textwidth}
    \centering
    \includegraphics[width=\linewidth]{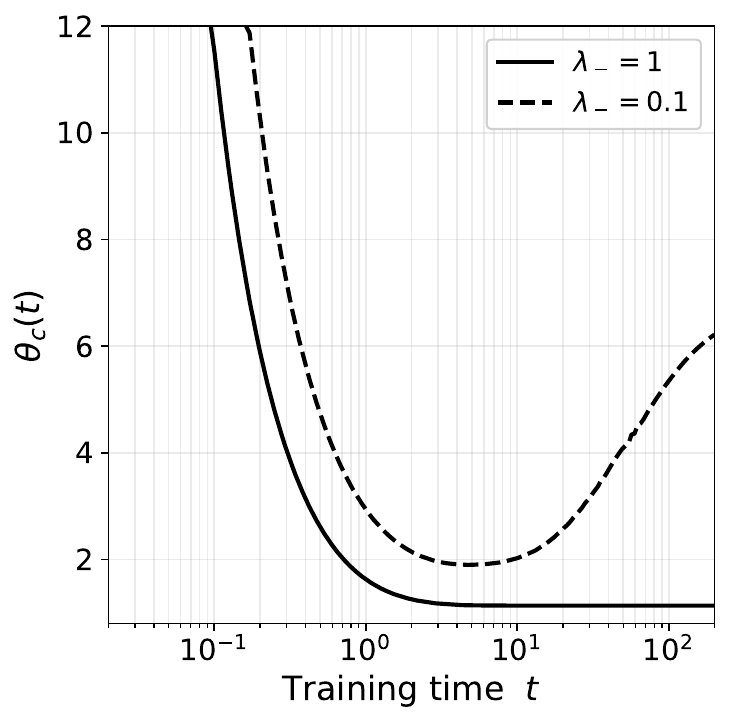}
    \captionof{figure}{\small Evolution of $\theta_c$ defined in~\eqref{eq:theta_c} for $\alpha=0.5$ and two values of $\lambda_-$.}
    \label{fig:theta_c}
  \end{minipage}
\par\medskip

\begin{remark}
    Conjecture~\ref{thm:three_phases} is hard to prove in full generality
    due to the rather complex functional form of $\theta_c$. Indeed, a proof
    would require showing that either the function is monotonically decreasing,
    or that it is quasi-convex (for instance if it has a unique global minimizer
    which is also the unique critical point), for whatever range of parameters.
    
    We can however state some partial results. For instance, we show in 
    Appendix~\ref{app:expandlambda} that 
    when $\lambda_-\to 1$ (i.e. in the isotropic limit), the function $\theta_c$
    becomes monotonically decreasing. This entails that the isotropic
    regime leads to either weak or strong signal regime depending on
    the value of $\theta$, as confirmed by
    Figure~\ref{fig:theta_c}.
    
\end{remark}

\subsection{Phase diagrams and the role of anisotropy}
\label{subsec:phase_diagrams}

The curve $\theta=\theta_c(t)$ in the $(\theta,t)$ plane directly separates the three regimes of Conjecture~\ref{thm:three_phases}. Horizontal cuts at fixed $\theta$ give the fitting set $\mathcal T_\theta$, and the minimum of the boundary curve corresponds to the minimum of $\theta_c(t)$ visible in Figure~\ref{fig:phase_thetat}. As the covariance becomes more anisotropic (smaller $\lambda_-$), the transient regime broadens; as the signal becomes stronger, the fitting window opens earlier and may become persistent.

\begin{figure}[ht]
  \centering
  \includegraphics[width=0.85\linewidth]{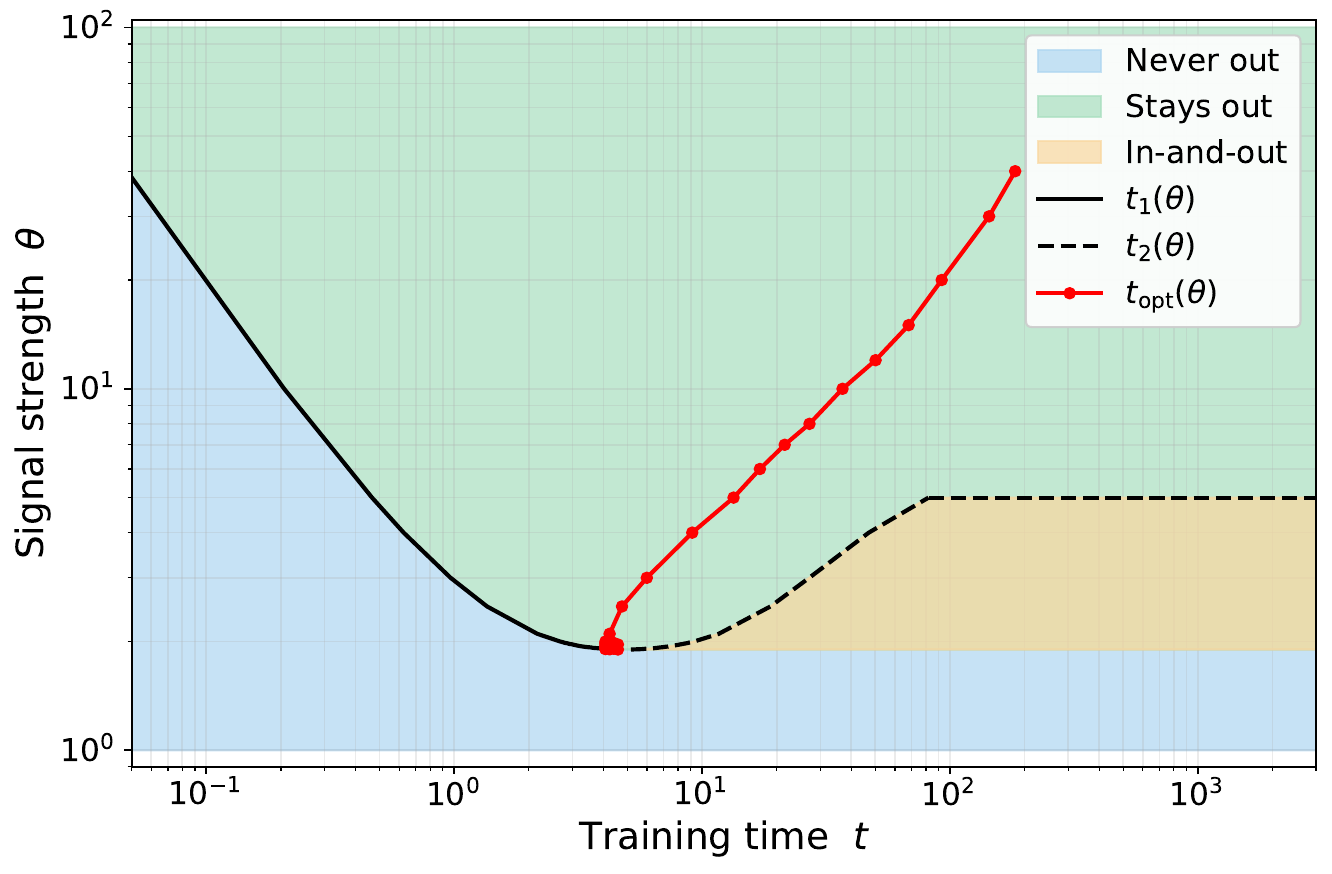}
  \caption{\small Phase diagram in the $(\theta,t)$ plane for $\lambda_-=0.1$, $\gamma = 1$ and $\alpha=0.5$. The critical curve $\theta=\theta_c(t)$ separates three regimes: no outlier (underfitting), transient outlier (early-stopping window), and persistent outlier. For a fixed intermediate value of $\theta$, the outlier emerges at $t_1$ and is reabsorbed at $t_2$. The overlaid red curve shows the optimal stopping time $t_{\rm opt}(\theta)$ defined in Section~\ref{subsec:overlap}: in our experiments it lies strictly inside the transient window $(t_1,t_2)$ whenever the latter exists.}
  \label{fig:phase_thetat}
\end{figure}

To examine how anisotropy controls the possible learning regimes, we scan the $(\theta, \lambda_-)$ plane and, for each pair, count the number of sign changes of
\[
F(\theta,t)
:=
\bigl(1-\theta\psi_t(L_+(t)^+)\bigr)^2 - \theta^2\phi_t(L_+(t)^+)\chi_t(L_+(t)^+)
\]
over $t\in[0.05,3000]$. One can derive this function through a simple manipulation of the
outlier condition~\eqref{eq:outlier}: a sign change of~$F$ shows by continuity that zero was crossed,
hence~\eqref{eq:outlier} had a solution.

\begin{figure}[ht]
  \centering
  \includegraphics[width=\linewidth]{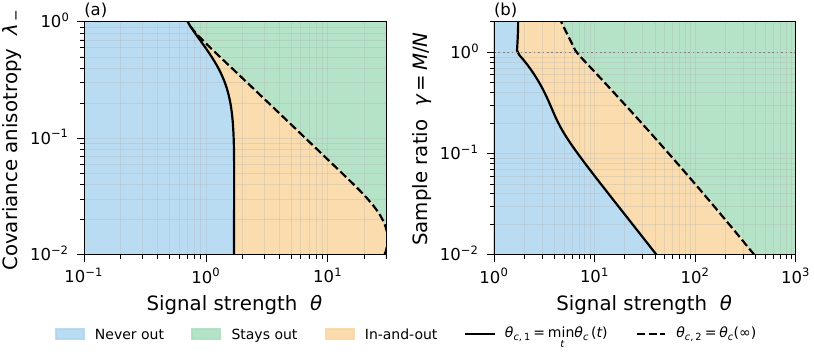}
  \caption{\small Outlier-regime classification obtained from $\theta_c(t)$, $t\in[0.05,3000]$ ($\alpha=0.5$, $\lambda_+=1$). (a) $(\theta,\lambda_-)$ plane at $\gamma=1$. (b) $(\theta,\gamma)$ plane at $\lambda_-=0.1$; dotted line: $\gamma=1$, below which a frozen null-space block appears (Appendix~\ref{app:gammalt1}). Blue: no outlier; green: persistent outlier; orange: transient outlier (early-stopping regime).}
  \label{fig:phase_lambda}
\end{figure}

Figure~\ref{fig:phase_lambda}(a) shows the resulting $(\theta,\lambda_-)$ diagram. The orange ``in-and-out'' wedge, the regime in which early stopping is the only way to recover the teacher, opens only below a critical anisotropy $\lambda_-^\star$: for nearly isotropic inputs only the weak- and strong-signal regimes survive. Even for strongly anisotropic inputs, the wedge is bounded below in $\theta$: below a minimum signal strength, no stopping time recovers the teacher. Panel (b) repeats the construction in the $(\theta,\gamma)$ plane, using the $3\times3$ Dyson system of Appendix~\ref{app:gammalt1} for $\gamma<1$. Fewer data samples widen the bulk, through the frozen null-space noise and through noisier least-squares directions, so all boundaries move to larger $\theta$: the best detection threshold grows from $1.69$ at $\gamma=1$ to about $41$ at $\gamma=0.01$, and at fixed $\theta$ the early-stopping window narrows and eventually closes. Note that, in the in-and-out case, we could also estimate the time interval over which the signal is detected.

\subsection{Teacher recovery and optimal stopping}
\label{subsec:overlap}

The regime classification of Conjecture~\ref{thm:three_phases} tells us \emph{when} the teacher is spectrally detectable, but not \emph{how well} it is reconstructed at any given time. The natural scalar summary of teacher recovery is the squared overlap
\[
q_t := (u_t^\star \cdot v)^2
\]
between the outlier eigenvector $u_t^\star$ and the true teacher direction $v$, with the convention $q_t=0$ when no outlier exists. This is the quantity that controls the rank-one content of the trained weight matrix along the teacher direction.

Because the outlier is a simple real zero of the $2\times 2$ Schur complement $M_t(z)$ introduced in Appendix~\ref{app:BBP}, the overlap $q_t$ can be obtained, via the Woodbury identity, as a residue of $v^\top(zI-S_t^\theta)^{-1}v$ at $z=\xi_{\theta,t}$. This residue is a rational function of the deterministic quantities $\phi_t,\psi_t,\chi_t$ and their derivatives at $\xi_{\theta,t}$; we sketch the derivation in Appendix~\ref{app:BBP} and do not expand the expression here, since it is not used in what follows. In the isotropic limit $\lambda_-=1$ it collapses to the classical Benaych-Georges--Nadakuditi formula for rank-one BBP. What matters for the discussion below is a qualitative feature of $q_t$: it is continuous on the fitting set $\mathcal T_\theta$, vanishes on its boundary (as the outlier merges into the bulk edge and the eigenvector delocalizes over the bulk), and is strictly positive in its interior. In the early-stopping regime $\mathcal T_\theta=(t_1,t_2)$, $q_t$ therefore attains at least one interior maximum, which defines an \emph{optimal early-stopping time}
\[
t_{\rm opt}(\theta) \in \underset{t\in(t_1,t_2)}{\arg\max}\; q_t.
\]

\paragraph{Numerical evidence} We computed $t_{\rm opt}(\theta)$ by one-dimensional maximization of $q_t$ over the window $(t_1(\theta),t_2(\theta))$ already located by the phase-diagram computation, for the full range of $\theta$ covered by Figure~\ref{fig:phase_thetat}. The result is overlaid as the red curve in Figure~\ref{fig:phase_thetat}. In our experiments $q_t$ is unimodal on $(t_1,t_2)$, so $t_{\rm opt}(\theta)$ is a single point, and it lies strictly inside the transient window for every $\theta$ in the early-stopping regime. The curve interpolates smoothly into the strong-signal regime, tracing the locus of best teacher recovery along the entire $(\theta,t)$ plane. An optimal stopping time also exists for a full-rank power-law covariance (Section~\ref{subsec:beyond}, Figure~\ref{fig:fullrank_overlap}). A rigorous decomposition of the test error in terms of $q_t$ is left to future work. 

\subsection{Beyond the solvable setting}
\label{subsec:beyond}

\paragraph{Nonlinear architectures} To test whether the transient BBP survives nonlinearity and depth, we trained two models on the same data and teacher (Appendix~\ref{app:nonlinear}) and tracked the spectrum of each symmetrized weight matrix. In a two-layer network $\widehat Y=W_2\tanh(W_1X)$, the scenario of Figure~\ref{fig:bulk} is reproduced in both layers, and synchronously: an outlier aligned with the teacher leaves the bulk of $W_1$ and of $W_2$ at the same time, and both are later reabsorbed as the bulk broadens (Figure~\ref{fig:mlp_bulk}). In a one-layer self-attention model $\widehat Y=(W_VX)\,\mathrm{softmax}\big(\,(W_KX)^\top(W_QX)\big)$, only the value matrix $W_V$, which enters the output linearly, shows the transient BBP (Figure~\ref{fig:attention_bulk}). The query and key matrices, which act through the softmax, do not: $W_Q$ carries only a weak signal that decays slowly, and the outliers of $W_K$ are orthogonal to the teacher, a reminder that a spectral outlier alone does not certify that the teacher has been learned.

\par\noindent
  \begin{minipage}[c]{0.55\textwidth}
    \paragraph{Beyond two blocks} Our theory relies on a two-block covariance, which gives the dynamics only two time scales. To test whether the transient BBP survives beyond this simplification, we simulated gradient flow with a full-rank covariance whose eigenvalues follow a power law $p(\lambda)\propto\lambda^{-1.5}$ on $[0.1,5]$ (Appendix~\ref{app:fullrank}), so that directions are learned over a continuum of time scales. No closed-form prediction is available here, yet the three regimes of Conjecture~\ref{thm:three_phases} persist (Figure~\ref{fig:fullrank_overlap}). A strong signal ($\theta=10$) is recovered and stays detectable, while a weak one ($\theta\lesssim1$, not shown) is never detected. For intermediate signals ($\theta=2,3$), the teacher is recovered only during a finite window: as in Section~\ref{subsec:overlap}, the overlap $q_t$ reaches an interior maximum, then returns to zero as the slow directions fit the noise. The transient BBP is therefore set by the spread of the covariance spectrum, i.e.\ its condition number, not by its two-block form.
  \end{minipage}%
  \hfill
  \begin{minipage}[c]{0.41\textwidth}
    \centering
    \includegraphics[width=\linewidth]{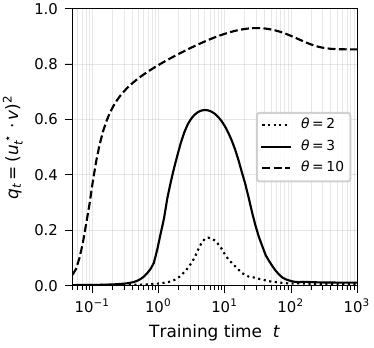}
    \captionof{figure}{\small Teacher overlap $q_t$ for a power-law spectrum ($\lambda\in[0.1,5]$, $N=400$).}
    \label{fig:fullrank_overlap}
  \end{minipage}
\par\medskip

\section{Conclusion and Extensions}

We introduced a solvable random-matrix model in which the spectrum of a symmetrized weight matrix undergoes a \emph{transient BBP transition} along gradient flow: depending on signal strength and covariance anisotropy, the teacher spike never emerges, emerges and persists, or exists only during a finite time window, in which the teacher overlap peaks at an optimal stopping time. The driving ingredient is anisotropy of the input covariance. Fast directions reveal the teacher early, whereas slow directions keep fitting noise and eventually hide the signal again, so overfitting is not an abstract late-time pathology here but a concrete consequence of spectral heterogeneity in the data.

\paragraph{Connection to empirical observations} Our model qualitatively reproduces features reported in empirical spectral studies of neural-network training \cite{martin2021implicit,thamm2022random,staats2023boundary,yunis2024approaching,olsen2025sgd}: isolated eigenvalues at intermediate training times, a broadening bulk at late times, and a best generalization time tied to the transient rather than the asymptotic spectrum. It does \emph{not} capture the heavy-tailed bulk and the multi-layer nonlinear coupling of deep networks at convergence.

%\paragraph{Extensions}
%Several extensions are natural. First, one can replace the rank-one teacher by a low-rank signal, which should lead to multiple interacting BBP transitions, where some directions may be permanently learned while others are fleeting (and others still remain unlearnable). Second, one can add stochastic gradient noise or model mismatch. Third, one can go beyond the strict two-block covariance assumption for inputs and ask what minimal conditions on the spectrum of the covariance matrix are needed for our transient BBP scenario to take place. Another direction is to study the Wishart-type multiplicative symmetrization $A_tA_t^\top$ rather than the additive one; this is the matrix on which empirical spectral analyses of deep networks are usually performed \cite{martin2021implicit}, and analyzing its time-dependent spectrum is a natural next step.

Finally, for infinitely wide networks the Neural Tangent Kernel \cite{jacot2018ntk,xiao2020disentangling} plays the role of $XX^\top$, so the transient BBP may be a useful toy model for early stopping in more complex architectures.

\newpage

\bibliographystyle{unsrt}
\bibliography{refs}

\begin{thebibliography}{10}

\bibitem{baik2005phase}
Jinho Baik, G{\'e}rard Ben~Arous, and Sandrine P{\'e}ch{\'e}.
\newblock Phase transition of the largest eigenvalue for nonnull complex sample
  covariance matrices.
\newblock {\em The Annals of Probability}, 33(5):1643--1697, 2005.

\bibitem{martin2021implicit}
Charles~H Martin and Michael~W Mahoney.
\newblock Implicit self-regularization in deep neural networks: Evidence from
  random matrix theory and implications for learning.
\newblock {\em Journal of Machine Learning Research}, 22(165):1--73, 2021.

\bibitem{thamm2022random}
Matthias Thamm, Max Staats, and Bernd Rosenow.
\newblock Random matrix analysis of deep neural network weight matrices.
\newblock {\em Physical Review E}, 106(5):054124, 2022.

\bibitem{staats2023boundary}
Max Staats, Matthias Thamm, and Bernd Rosenow.
\newblock Boundary between noise and information applied to filtering neural
  network weight matrices.
\newblock {\em Physical Review E}, 108(2):L022302, 2023.

\bibitem{yunis2024approaching}
David Yunis, Kumar~Kshitij Patel, Samuel Wheeler, Pedro Savarese, Gal Vardi,
  Karen Livescu, Michael Maire, and Matthew~R Walter.
\newblock Approaching deep learning through the spectral dynamics of weights.
\newblock {\em arXiv preprint arXiv:2408.11804}, 2024.

\bibitem{edwards1976eigenvalue}
Samuel~F Edwards and Raymund~C Jones.
\newblock The eigenvalue spectrum of a large symmetric random matrix.
\newblock {\em Journal of Physics A: Mathematical and General},
  9(10):1595--1603, 1976.

\bibitem{benaych2011eigenvalues}
Florent Benaych-Georges and Raj~Rao Nadakuditi.
\newblock The eigenvalues and eigenvectors of finite, low rank perturbations of
  large random matrices.
\newblock {\em Advances in Mathematics}, 227(1):494--521, 2011.

\bibitem{saxe2014exact}
Andrew~M. Saxe, James~L. McClelland, and Surya Ganguli.
\newblock Exact solutions to the nonlinear dynamics of learning in deep linear
  neural networks.
\newblock In {\em International Conference on Learning Representations}, 2014.

\bibitem{advani2020highdim}
Madhu~S. Advani, Andrew~M. Saxe, and Haim Sompolinsky.
\newblock High-dimensional dynamics of generalization error in neural networks.
\newblock {\em Neural Networks}, 132:428--446, 2020.

\bibitem{ali2019continuous}
Alnur Ali, J.~Zico Kolter, and Ryan~J. Tibshirani.
\newblock A continuous-time view of early stopping for least squares
  regression.
\newblock In {\em Proceedings of the Twenty-Second International Conference on
  Artificial Intelligence and Statistics}, volume~89 of {\em Proceedings of
  Machine Learning Research}, pages 1370--1378. PMLR, 2019.

\bibitem{lampinen2019analytic}
Andrew~K. Lampinen and Surya Ganguli.
\newblock An analytic theory of generalization dynamics and transfer learning
  in deep linear networks.
\newblock In {\em International Conference on Learning Representations}, 2019.

\bibitem{watanabe2026anisotropic}
Taishi Watanabe, Ryo Karakida, and Jun-nosuke Teramae.
\newblock The impact of anisotropic covariance structure on the training
  dynamics and generalization error of linear networks.
\newblock {\em arXiv preprint arXiv:2601.06961}, 2026.

\bibitem{canatar2021spectral}
Abdulkadir Canatar, Blake Bordelon, and Cengiz Pehlevan.
\newblock Spectral bias and task-model alignment explain generalization in
  kernel regression and infinitely wide neural networks.
\newblock {\em Nature Communications}, 12(1):2914, 2021.

\bibitem{paquette2025homogenization}
Courtney Paquette, Elliot Paquette, Ben Adlam, and Jeffrey Pennington.
\newblock Homogenization of {SGD} in high-dimensions: Exact dynamics and
  generalization properties.
\newblock {\em Mathematical Programming}, 214(1):1--90, 2025.

\bibitem{collinswoodfin2024hitting}
Elizabeth Collins-Woodfin, Courtney Paquette, Elliot Paquette, and Inbar
  Seroussi.
\newblock Hitting the high-dimensional notes: An {ODE} for {SGD} learning
  dynamics on {GLMs} and multi-index models.
\newblock {\em Information and Inference: A Journal of the IMA}, 13(4):iaae028,
  2024.

\bibitem{olsen2025sgd}
Brian~Richard Olsen, Sam Fatehmanesh, Frank Xiao, Adarsh Kumarappan, and
  Anirudh Gajula.
\newblock From sgd to spectra: A theory of neural network weight dynamics.
\newblock {\em arXiv preprint arXiv:2507.12709}, 2025.

\bibitem{aarts2025dyson}
Gert Aarts, Biagio Lucini, and Chanju Park.
\newblock Stochastic weight matrix dynamics during learning and {D}yson
  {B}rownian motion.
\newblock {\em Physical Review E}, 111:015303, 2025.

\bibitem{ba2022high}
Jimmy Ba, Murat~A Erdogdu, Taiji Suzuki, Zhichao Wang, Denny Wu, and Greg Yang.
\newblock High-dimensional asymptotics of feature learning: How one gradient
  step improves the representation.
\newblock In {\em Advances in Neural Information Processing Systems},
  volume~35, 2022.

\bibitem{moniri2024theory}
Behrad Moniri, Donghwan Lee, Hamed Hassani, and Edgar Dobriban.
\newblock A theory of non-linear feature learning with one gradient step in
  two-layer neural networks.
\newblock In {\em International Conference on Machine Learning}, pages
  36106--36159. PMLR, 2024.

\bibitem{dandi2025random}
Yatin Dandi, Luca Pesce, Hugo Cui, Florent Krzakala, Yue~M Lu, and Bruno
  Loureiro.
\newblock A random matrix theory perspective on the spectrum of learned
  features and asymptotic generalization capabilities.
\newblock In {\em International Conference on Artificial Intelligence and
  Statistics}, pages 2224--2232. PMLR, 2025.

\bibitem{defilippis2026scaling}
Leonardo Defilippis, Yizhou Xu, Julius Girardin, Emanuele Troiani, Vittorio
  Erba, Lenka Zdeborov{\'a}, Bruno Loureiro, and Florent Krzakala.
\newblock Scaling laws and spectra of shallow neural networks in the feature
  learning regime.
\newblock In {\em International Conference on Learning Representations}, 2026.

\bibitem{defilippis2026optimal}
Leonardo Defilippis, Florent Krzakala, Bruno Loureiro, and Antoine Maillard.
\newblock Optimal scaling laws in learning hierarchical multi-index models.
\newblock In {\em Advances in Neural Information Processing Systems}, 2026.

\bibitem{benarous2025spectral}
G{\'e}rard Ben~Arous, Reza Gheissari, Jiaoyang Huang, and Aukosh Jagannath.
\newblock Spectral alignment of stochastic gradient descent for
  high-dimensional classification tasks.
\newblock {\em The Annals of Applied Probability}, 35(4):2767--2822, 2025.

\bibitem{benarous2026local}
G{\'e}rard Ben~Arous, Reza Gheissari, Jiaoyang Huang, and Aukosh Jagannath.
\newblock Local geometry of high-dimensional mixture models: Effective spectral
  theory and dynamical transitions.
\newblock {\em The Annals of Statistics}, 54(4):1682--1711, 2026.

\bibitem{bonnaire2025role}
Tony Bonnaire, Giulio Biroli, and Chiara Cammarota.
\newblock The role of the time-dependent hessian in high-dimensional
  optimization.
\newblock {\em Journal of Statistical Mechanics: Theory and Experiment},
  2025(8):083401, 2025.

\bibitem{montanari2026phase}
Andrea Montanari and Zihao Wang.
\newblock Phase transitions for feature learning in neural networks.
\newblock {\em arXiv preprint arXiv:2602.01434}, 2026.

\bibitem{barbier2019optimal}
Jean Barbier, Florent Krzakala, Nicolas Macris, L{\'e}o Miolane, and Lenka
  Zdeborov{\'a}.
\newblock Optimal errors and phase transitions in high-dimensional generalized
  linear models.
\newblock {\em Proceedings of the National Academy of Sciences},
  116(12):5451--5460, 2019.

\bibitem{ajanki2017universality}
Oskari~H Ajanki, L{\'a}szl{\'o} Erd{\H{o}}s, and Torben Kr{\"u}ger.
\newblock Universality for general {W}igner-type matrices.
\newblock {\em Probability Theory and Related Fields}, 169(3):667--727, 2017.

\bibitem{ajanki2019}
Oskari~H Ajanki, L{\'a}szl{\'o} Erd{\H{o}}s, and Torben Kr{\"u}ger.
\newblock Stability of the matrix {D}yson equation and random matrices with
  correlations.
\newblock {\em Probability Theory and Related Fields}, 173(1):293--373, 2019.

\bibitem{maillard2019high}
Antoine Maillard, Laura Foini, Alejandro~Lage Castellanos, Florent Krzakala,
  Marc M{\'e}zard, and Lenka Zdeborov{\'a}.
\newblock High-temperature expansions and message passing algorithms.
\newblock {\em Journal of Statistical Mechanics: Theory and Experiment},
  2019(11):113301, 2019.

\bibitem{atanasov2024scaling}
Alexander Atanasov, Jacob~A Zavatone-Veth, and Cengiz Pehlevan.
\newblock Scaling and renormalization in high-dimensional regression.
\newblock {\em arXiv preprint arXiv:2405.00592}, 2024.

\bibitem{bordelon2026disordered}
Blake Bordelon and Cengiz Pehlevan.
\newblock Disordered dynamics in high dimensions: Connections to random
  matrices and machine learning.
\newblock {\em arXiv preprint arXiv:2601.01010}, 2026.

\bibitem{liao2025rmt}
Zhenyu Liao and Michael~W. Mahoney.
\newblock Random matrix theory for deep learning: Beyond eigenvalues of linear
  models.
\newblock {\em IEEE Signal Processing Magazine}, 2025.
\newblock To appear.

\bibitem{potters2020first}
Marc Potters and Jean-Philippe Bouchaud.
\newblock {\em A first course in random matrix theory: for physicists,
  engineers and data scientists}.
\newblock Cambridge University Press, 2020.

\bibitem{jacot2018ntk}
Arthur Jacot, Franck Gabriel, and Cl{\'e}ment Hongler.
\newblock Neural tangent kernel: Convergence and generalization in neural
  networks.
\newblock In {\em Advances in Neural Information Processing Systems},
  volume~31, pages 8571--8580, 2018.

\bibitem{xiao2020disentangling}
Lechao Xiao, Jeffrey Pennington, and Samuel~S Schoenholz.
\newblock Disentangling trainability and generalization in deep neural
  networks.
\newblock In {\em Proceedings of the 37th International Conference on Machine
  Learning}, pages 10462--10472. PMLR, 2020.

\end{thebibliography}

\newpage
\appendix

\section{Proof of Proposition~\ref{prop:dyson}}
\label{app:proof_dyson}

\begin{proof}
We fix a time $t\geq 0$ and drop all time indexes in the proof for notational simplicity. Write
\[
r_i := (D_1)_{ii}, \qquad q_i := (D_2)_{ii},
\]
so that $r_i\in\{a,b\}$ and $q_i\in\{c,d\}$ depending on whether $i$ belongs to block $A$ or block $B$.
From \eqref{eq:S_def}, for $i<j$ we have
\[
S_{ij}=(r_i+r_j)W_{ij}+q_j\Gamma_{ij}+q_i\Gamma_{ji}.
\]
Since the families $\{W_{ij}\}_{i\le j}$ and $\{\Gamma_{ij}\}_{i,j}$ are independent and centered, the upper-triangular entries of $S$ are independent and centered as well. Moreover,
\[
N\,\E S_{ij}^2
=
(r_i+r_j)^2
+
\frac{N}{M}\,(q_i^2+q_j^2)
\;\longrightarrow\;
(r_i+r_j)^2+\frac{q_i^2+q_j^2}{\gamma}.
\]
Hence the variance profile is piecewise constant, and for off-diagonal entries it is exactly
\[
\sigma_{AA}=4a^2+\frac{2c^2}{\gamma},
\qquad
\sigma_{BB}=4b^2+\frac{2d^2}{\gamma},
\qquad
\sigma_{AB}=(a+b)^2+\frac{c^2+d^2}{\gamma},
\]
as claimed in \eqref{eq:variance_profile}. The diagonal entries also have variance of order $N^{-1}$ and fit into the same Wigner-type framework.

Now consider the standard resolvent
\[
G(z):=(S-zI)^{-1}, \qquad z\in\mathbb C^+.
\]
By the general theory of Wigner-type matrices with variance profile
$\mathfrak s_{ij}:=\E |S_{ij}|^2$,
the diagonal resolvent entries converge to deterministic limits
$\widehat m_i(z)$ solving the quadratic vector equation
\[
-\frac{1}{\widehat m_i(z)}
=
z+\sum_{j=1}^N \mathfrak s_{ij}\,\widehat m_j(z),
\]
with $\Im \widehat m_i(z)>0$ for $\Im z>0$
\cite{ajanki2019,ajanki2017universality}.

Since the variance profile has only two block values, the solution is constant on each block:
\[
\widehat m_i(z)=\widehat m_A(z)\quad (i\in A),\qquad
\widehat m_i(z)=\widehat m_B(z)\quad (i\in B).
\]
Therefore the quadratic vector equation reduces to
\[
-\frac{1}{\widehat m_A(z)}
=
z+\alpha\,\sigma_{AA}\,\widehat m_A(z)
+(1-\alpha)\,\sigma_{AB}\,\widehat m_B(z),
\]
\[
-\frac{1}{\widehat m_B(z)}
=
z+\alpha\,\sigma_{AB}\,\widehat m_A(z)
+(1-\alpha)\,\sigma_{BB}\,\widehat m_B(z).
\]

Finally, Proposition~\ref{prop:dyson} is written in terms of the resolvent
\[
(zI-S)^{-1}=-\, (S-zI)^{-1}=-G(z).
\]
Thus, setting
\[
m_A(z):=-\widehat m_A(z),\qquad m_B(z):=-\widehat m_B(z),
\]
we obtain exactly
\[
\frac{1}{m_A(z)}
=
z-\alpha\,\sigma_{AA}\,m_A(z)-(1-\alpha)\,\sigma_{AB}\,m_B(z),
\]
\[
\frac{1}{m_B(z)}
=
z-\alpha\,\sigma_{AB}\,m_A(z)-(1-\alpha)\,\sigma_{BB}\,m_B(z).
\]
Moreover $\Im m_A(z),\Im m_B(z)<0$ for $\Im z>0$, and
$m_A(z),m_B(z)\sim 1/z$ as $|z|\to\infty$.
The averaged limit
\[
m(z)=\alpha\,m_A(z)+(1-\alpha)\,m_B(z)
\]
is therefore the Stieltjes transform of the limiting empirical spectral measure in the convention of \eqref{eq:mz}, and the density is recovered by the inversion formula
\[
\rho_S(\lambda)
=
-\frac{1}{\pi}\lim_{\eta\downarrow 0}\Im\,m(\lambda+i\eta).
\]
This proves the proposition.
\end{proof}

\section{The case $\gamma < 1$: A $3 \times 3$ Dyson system}
\label{app:gammalt1}

When $M < N$ (i.e.\ $\gamma = M/N < 1$), the input covariance
$XX^\top$ has a null space of dimension $(1-\gamma)N$.  In the null
space, the gradient flow never moves: the exponential
$e^{-tXX^\top}$ acts as the identity, and the noise term
$ZX^\top(XX^\top)^{-1}(I - e^{-tXX^\top})$ vanishes (since
$X^\top$ annihilates the null space).  Consequently, the weight
matrix retains its random initialisation in those directions for all
time.

This introduces a third block in the variance profile of the
symmetrised noise matrix, leading to a $3\times 3$ Dyson system
that generalises the $2\times 2$ system of
Section~\ref{sec:spectrum}. We also drop the time index for simplicity.

\subsection{Three-block structure}

We split the $N$ indices into three classes:
\begin{center}
\begin{tabular}{llll}
\toprule
Block & Fraction & $a_p$ & $c_p$ \\
\midrule
A (fast)  & $\alpha_A = \gamma/2$   & $e^{-t}$
          & $1 - e^{-t}$ \\[3pt]
B (slow)  & $\alpha_B = \gamma/2$   & $e^{-\lambda_-^2 t}$
          & $(1 - e^{-\lambda_-^2 t})/\lambda_-$ \\[3pt]
C (null)  & $\alpha_C = 1 - \gamma$ & $1$ & $0$ \\
\bottomrule
\end{tabular}
\end{center}

\noindent
The null-space block C has $a_C = 1$ (the initialisation is never
damped) and $c_C = 0$ (no noise is injected through $X$).

Using the general variance formula
$\sigma_{pq} = (a_p + a_q)^2 + (c_p^2 + c_q^2)/\gamma$, we obtain
six variance parameters:
\begin{equation}
\label{eq:var-3block}
\begin{aligned}
  \sigma_{AA} &= 4a_A^2 + \frac{2c_A^2}{\gamma}, \\[4pt]
  \sigma_{BB} &= 4a_B^2 + \frac{2c_B^2}{\gamma}, \\[4pt]
  \sigma_{CC} &= 4, \\[4pt]
  \sigma_{AB} &= (a_A + a_B)^2 + \frac{c_A^2 + c_B^2}{\gamma}, \\[4pt]
  \sigma_{AC} &= (1 + a_A)^2 + \frac{c_A^2}{\gamma}, \\[4pt]
  \sigma_{BC} &= (1 + a_B)^2 + \frac{c_B^2}{\gamma}.
\end{aligned}
\end{equation}

Note that $\sigma_{CC} = 4$ is time-independent: the null-space block
is a frozen GOE contribution at all times.

\subsection{The $3\times 3$ Dyson equation}

The partial Stieltjes transforms $m_A(z)$, $m_B(z)$, $m_C(z)$
satisfy:
\begin{equation}
\label{eq:dyson-3x3}
\begin{cases}
  \displaystyle \frac{1}{m_A(z)}
  = z
  - \frac{\gamma}{2}\,\sigma_{AA}\,m_A
  - \frac{\gamma}{2}\,\sigma_{AB}\,m_B
  - (1-\gamma)\,\sigma_{AC}\,m_C, \\[10pt]
  \displaystyle \frac{1}{m_B(z)}
  = z
  - \frac{\gamma}{2}\,\sigma_{AB}\,m_A
  - \frac{\gamma}{2}\,\sigma_{BB}\,m_B
  - (1-\gamma)\,\sigma_{BC}\,m_C, \\[10pt]
  \displaystyle \frac{1}{m_C(z)}
  = z
  - \frac{\gamma}{2}\,\sigma_{AC}\,m_A
  - \frac{\gamma}{2}\,\sigma_{BC}\,m_B
  - (1-\gamma)\,\sigma_{CC}\,m_C,
\end{cases}
\end{equation}
with the Stieltjes-branch conditions $\Im\,m_p(z) < 0$ for
$\Im\,z > 0$ and $m_p(z) \sim 1/z$ as $|z| \to \infty$.

The global Stieltjes transform and spectral density are
\begin{equation}
\label{eq:m-3block}
  m(z) = \frac{\gamma}{2}\,m_A(z)
       + \frac{\gamma}{2}\,m_B(z)
       + (1-\gamma)\,m_C(z),
  \qquad
  \rho_S(\lambda)
  = -\frac{1}{\pi}\lim_{\eta\downarrow 0}\Im\,m(\lambda + i\eta).
\end{equation}

\subsection{Effect of the null space on the transient BBP window}

Compared to the oversampled regime $\gamma > 1$, the null-space
block C introduces a frozen GOE component of variance $\sigma_{CC}=4$
that competes with the signal at all times.  This has two
consequences:
\begin{enumerate}
  \item The bulk spectrum is generically wider (since the frozen
    initialisation noise never decays), making it harder for an
    outlier to emerge.
  \item The null-space directions do not contribute to $\psi$ or
    $\chi$ (since they carry no learned signal), but they do
    contribute to $\phi$ (through $m_C$), which enters the
    right-hand side of the outlier equation.
\end{enumerate}

As a result, the transient BBP window $[t_1, t_2]$ is generically
\emph{narrower} for $\gamma < 1$ than for $\gamma > 1$: more
directions carry permanent noise without contributing to learning,
raising the effective detection threshold $\theta_c(t)$. Figure~\ref{fig:phase_lambda}(b) in Section~\ref{subsec:phase_diagrams} quantifies this effect over $\gamma\in[0.01,2]$.

\section{Outlier equation and BBP transition}
\label{app:BBP}

\subsection{Weighted finite-rank teacher perturbation}

In the covariance eigenbasis, the signal contribution to the symmetrized matrix is
\[
\theta\bigl(\widetilde v w_t^\top + w_t\widetilde v^\top\bigr),
\qquad
w_t=(I-e^{-t\Lambda^2})\widetilde v.
\]
This perturbation is rank two unless $w_t$ is proportional to $\widetilde v$, which happens only in the isotropic case. Thus the standard rank-one BBP formula does not directly apply at finite time \cite{benaych2011eigenvalues}.

\subsection{Rank-two determinant and outlier condition}
\label{subsec:rank2det}

To simplify notation, we drop the tildes and the time subscript and write the perturbed matrix as
\[
S^\theta = S + \theta(wv^\top + vw^\top).
\]
Introduce
\[
U=[w,v]\in\mathbb R^{N\times 2},
\qquad
V=[v,w]\in\mathbb R^{N\times 2},
\]
and let
\[
R(z)=(zI-S)^{-1}.
\]
By the matrix determinant lemma,
\[
\det(zI-S^\theta)
=
\det(zI-S)\,\det\!\bigl(I_2-\theta V^\top R(z)U\bigr).
\]

The $2\times2$ Schur complement is
\[
M(z)
=
I_2-\theta
\begin{pmatrix}
v^\top R(z)w & v^\top R(z)v\\
w^\top R(z)w & w^\top R(z)v
\end{pmatrix}
=
\begin{pmatrix}
1-\theta\psi(z) & -\theta\phi_t(z)\\
-\theta\chi(z) & 1-\theta\psi(z)
\end{pmatrix},
\]
where
\begin{equation}
\label{eq:qforms}
\phi_t(z)=v^\top R(z)v,
\qquad
\psi_t(z)=v^\top R(z)w,
\qquad
\chi_t(z)=w^\top R(z)w.
\end{equation}

\paragraph{Concentration of quadratic forms}
Since $v$ is asymptotically isotropic in the eigenbasis of the noise matrix and the two blocks have deterministic weights, the quadratic forms concentrate according to
\begin{align}
\phi_t(z) &\xrightarrow[N\to\infty]{} \alpha\,m_{A,t}(z) + (1-\alpha)\,m_{B,t}(z), \notag\\
\psi_t(z) &\xrightarrow[N\to\infty]{} \alpha\,f_{A,t}\,m_{A,t}(z) + (1-\alpha)\,f_{B,t}\,m_{B,t}(z), \label{eq:quad_form}\\
\chi_t(z) &\xrightarrow[N\to\infty]{} \alpha\,f_{A,t}^2\,m_{A,t}(z) + (1-\alpha)\,f_{B,t}^2\,m_{B,t}(z), \notag
\end{align}
with
\[
f_{A,t}=1-e^{- \lambda_+^2t},
\qquad
f_{B,t}=1-e^{-\lambda_-^2 t}.
\]

\paragraph{Outlier equation}
Eigenvalues of $S^\theta$ outside the bulk correspond to real zeros of $\det M(z)$. Computing the determinant gives
\begin{equation}
\label{eq:outlier_det}
\det M(z)
=
\bigl(1-\theta\psi_t(z)\bigr)^2 - \theta^2\phi_t(z)\chi_t(z)=0,
\end{equation}
which is the outlier equation used in Proposition~\ref{prop:outlier}.

\paragraph{Eigenvector overlap with the teacher}
Given a simple outlier root $\xi_\theta$ of~\eqref{eq:outlier_det}, the projections of the associated eigenvector $u^\star$ of $S^\theta$ onto $v$ and $w$ concentrate to deterministic limits, while the component orthogonal to $\mathrm{span}(v,w)$ is delocalized over the bulk. The squared overlap $q_t=(u^\star\cdot v)^2$ with the teacher direction is obtained, via the Woodbury identity, as the residue of the resolvent $v^\top(zI-S^\theta)^{-1}v$ at $z=\xi_\theta$; this reduces to differentiating $\det M(z)$ at $\xi_\theta$ and yields a rational function of $\phi_t,\psi_t,\chi_t$ and their derivatives at $\xi_\theta$. We omit the explicit expression, which is not used in the main text beyond the qualitative properties of $q_t$ discussed in Section~\ref{subsec:overlap}.

\section{Fixed-point algorithm for the Dyson equation}
\label{app:fixedpoint}

Proposition~\ref{prop:dyson} characterizes $(m_A(z),m_B(z))$ as the unique physical solution of \eqref{eq:dyson} with $\Im m_A(z),\Im m_B(z)<0$ for $\Im z>0$. Numerically, we solve this system by fixed-point iteration, initializing $m_A^{(0)}=m_B^{(0)}=1/z$ and iterating
\begin{equation}
\label{eq:fp_iter}
\begin{aligned}
m_A^{(n+1)}
&=
\frac{1}{z-\alpha\sigma_{AA}m_A^{(n)}-(1-\alpha)\sigma_{AB}m_B^{(n)}},
\\[4pt]
m_B^{(n+1)}
&
=
\frac{1}{z-\alpha\sigma_{AB}m_A^{(n)}-(1-\alpha)\sigma_{BB}m_B^{(n)}}.
\end{aligned}
\end{equation}

\paragraph{Numerical implementation}
For $z=x+i\varepsilon$ with $\varepsilon>0$, the iteration preserves the physical branch and converges robustly in our experiments. In practice we use $\varepsilon=10^{-2}$ and stop when the $\ell^\infty$ change between successive iterates is below $10^{-10}$. The Dyson-equation solutions, the linear simulations ($N\le 500$, up to 30 realizations) and the phase-diagram scans run on a standard CPU; the nonlinear experiments of Appendix~\ref{app:nonlinear} use a single GPU.

\paragraph{Bulk edge and phase boundary}
The upper edge $L_+$ is estimated from the rightmost point on a fine real grid where the numerical density exceeds a small threshold. For each fixed $\theta$, we then evaluate
\[
F(\theta)
=
\bigl(1-\theta\psi(L_+^+)\bigr)^2 - \theta^2\phi(L_+^+)\chi(L_+^+)
\]
Doing this on a coarse grid in $t$ and refining each sign change by bisection yields the transition times and the phase diagrams shown in the main text.

\section{Isotropic expansion of $\theta_c$}
\label{app:expandlambda}
When $\lambda_-\to 1$, we can show via standard Taylor expansion that
\[ 
\theta_c(t) = \theta_{\mathrm{iso}}(t)\Big[1 + (1-\alpha)(1-\lambda_-)\,G(t)\Big]
+ O\big((1-\lambda_-)^2\big),
\]
where
\[
\theta_{\mathrm{iso}}(t) = \frac{\sqrt{(e^t-1)^2 + 2\gamma}}{\sqrt{2\gamma}\,(e^t-1)},
\qquad
G(t) = \frac{(e^t-1)^3 + 4\gamma\, t\, e^{t}}{(e^t-1)\big[(e^t-1)^2 + 2\gamma\big]}.
\]
Simple differentiation shows that $\theta_{\mathrm{iso}}$ is 
decreasing monotonically. On the other hand, a standard asymptotic calculus leads to
\begin{equation}
\label{eq:limitbehav}
\theta_c(t) \underset{t\to 0}{\sim} \frac{1 + 2(1-\alpha)(1-\lambda_-)}{t},
\qquad
\theta_c(t) \xrightarrow[t\to\infty]{} \frac{1 + (1-\alpha)(1-\lambda_-)}{\sqrt{2\gamma}},
\end{equation}
entailing that $G$ grows at a similar scale as $\theta_{\mathrm{iso}}$ around
zero. This shows that $\theta_c$ is indeed monotonically decreasing from infinity
to the limiting value on the right hand side of~\eqref{eq:limitbehav},
in accordance with our expectations and Figure~\ref{fig:theta_c}.

\section[Full-rank power-law spectrum]{Full-rank power-law spectrum: numerical validation}
\label{app:fullrank}

The two-block covariance model of Section~\ref{sec:twoblock} is analytically
tractable because it reduces the Dyson system to two equations.  Here we check
numerically that the transient BBP phenomenology is not an artefact of that
simplification, by replacing the two-block spectrum with a continuous power-law
distribution.

\paragraph{Setup}
We draw the $N$ singular values of the input matrix $X$ independently from the
truncated power-law density
\[
  p(\lambda) \propto \lambda^{-\beta}, \qquad \lambda \in [\lambda_{\min},\lambda_{\max}],
\]
with $\beta = 1.5$, $\lambda_{\min} = 0.1$, $\lambda_{\max} = 5$, and $N=400$.
The rest of the setup follows Section~\ref{sec:model} exactly: the teacher is
rank-one with signal amplitude $\theta$, the initialisation is GOE at scale
$1/\sqrt{N}$, and the noise matrix has i.i.d.\ $\mathcal{N}(0,1/N)$ entries.
Working in the eigenbasis of $XX^\top$ (with eigenvalues $\mu_i = \lambda_i^2$),
the closed-form gradient-flow solution~\eqref{eq:At} becomes
\[
  A_t = Q\!\left(e^{-t\mu}\cdot G \;+\; \frac{1-e^{-t\mu}}{\mu}\cdot W_Z\right),
\]
where $Q$ is the eigenvector matrix of $XX^\top$, $\mu =
\mathrm{diag}(\mu_1,\dots,\mu_N)$, $G = Q^\top A_{\mathrm{init}}$, and $W_Z =
Q^\top(X Y_0^\top)$ encodes the teacher-plus-noise contribution.  Simulations
are run over 30 independent realisations at each of 85 logarithmically spaced
times $t \in [0.05, 1000]$ and averaged.

\paragraph{Results}
Figure~\ref{fig:fullrank_overlap} (Section~\ref{subsec:beyond}) shows the squared overlap
$q_t = (u_t^\star \cdot v)^2$ between the top eigenvector of
$S_t = A_t + A_t^\top$ and the teacher direction $v$, for $\theta=2,3$
(transient) and $\theta=10$ (persistent).

The three qualitative regimes of Conjecture~\ref{thm:three_phases} are all
visible: for $\theta \lesssim 1$ the overlap never exceeds the random baseline
$1/N$ (weak-signal regime); for intermediate $\theta$ the overlap rises,
attains an interior maximum, and decays back to baseline (early-stopping
regime); and for large $\theta$ the overlap rises and remains elevated
throughout training (strong-signal regime).  The peak overlap increases with
$\theta$ and the transient window widens, consistent with the phase-diagram
picture of Section~\ref{subsec:phase_diagrams}.  These observations confirm
that the transient BBP mechanism is robust to the specific form of the
covariance spectrum and is not a peculiarity of the two-block simplification.

\keep[22\baselineskip]\section{Nonlinear architectures}
\label{app:nonlinear}

\paragraph{Setup}
We use the data and teacher of the main text: $N=M=500$, $\alpha=0.5$, $\lambda_-=0.1$, $\gamma=1$, and $Y=\theta vv^\top X+Z$. Models are trained on $\frac12\|\widehat Y-Y\|_F^2$ by full-batch gradient descent, with time $t=k\eta$ after $k$ steps of size $\eta$. All square weights start from independent GOE matrices, and results are averaged over 30 realizations at $\theta=15$. For every weight $W$ we compute the spectrum of $W+W^\top$; an eigenvalue is declared an outlier when its mean over realizations exceeds the corresponding extreme eigenvalue obtained without signal ($\theta=0$) by more than $\max(4\,\mathrm{std},0.15)$.
\begin{itemize}
  \item \textbf{Two-layer network:} $\widehat Y=W_2\tanh(W_1X)$, $\eta=0.015$, $t\le200$.
  \item \textbf{One-layer self-attention:} $\widehat Y=(W_VX)\,\mathrm{softmax}\big(s\,(W_KX)^\top(W_QX)\big)$, softmax over keys, with the samples as the attention set, $\eta=0.01$, $t\le300$. The temperature $s$ is fixed at initialization so that the scores have standard deviation $0.5$; with the usual $1/\sqrt N$ scaling the softmax stays essentially uniform and nothing is learned on this time scale.
\end{itemize}

\paragraph{Two-layer network}
Figure~\ref{fig:mlp_bulk} is the analogue of Figure~\ref{fig:bulk}. In both layers a pair of outliers leaves the bulk at $t\approx0.5$, is still present at $t=7$, and is reabsorbed by $t=20$ as the bulk broadens, while the bulk itself deforms differently in the two layers. At $t=7$ the top eigenvector has overlap $(u^\star\cdot v)^2\approx0.34$ with the teacher for $W_1$ and $0.38$ for $W_2$.

\begin{figure}[ht]
  \centering
  \includegraphics[width=\linewidth]{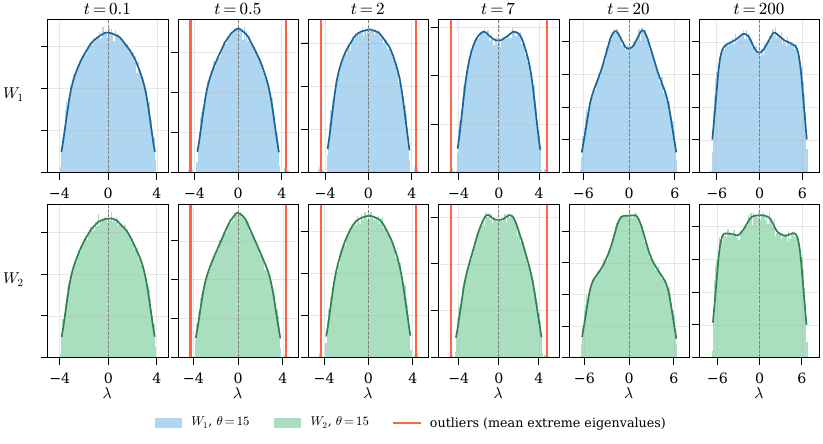}
  \caption{\small Two-layer $\tanh$ network, $\theta=15$: spectral density of $W_\ell+W_\ell^\top$ for each layer (rows) at six training times (30 realizations). Red vertical lines: mean largest and smallest eigenvalues, shown when they separate from those obtained without signal ($\theta=0$). As in Figure~\ref{fig:bulk}, the outliers emerge from the bulk and are then reabsorbed.}
  \label{fig:mlp_bulk}
\end{figure}

\paragraph{Self-attention}
The three weights behave differently (Figure~\ref{fig:attention_bulk}). The value matrix $W_V$ reproduces the transient behaviour: its outliers are present from $t=2$ to $t=50$ and reabsorbed by $t=300$, with overlap $\approx0.47$ at $t=10$. The query matrix $W_Q$ carries a weaker signal (overlap $\approx0.23$) that decays slowly. The key matrix $W_K$ also develops outliers, but their overlap with $v$ stays at the random level $1/N$: a spectral outlier alone does not certify that the teacher has been learned.

\begin{figure}[ht]
  \centering
  \includegraphics[width=\linewidth]{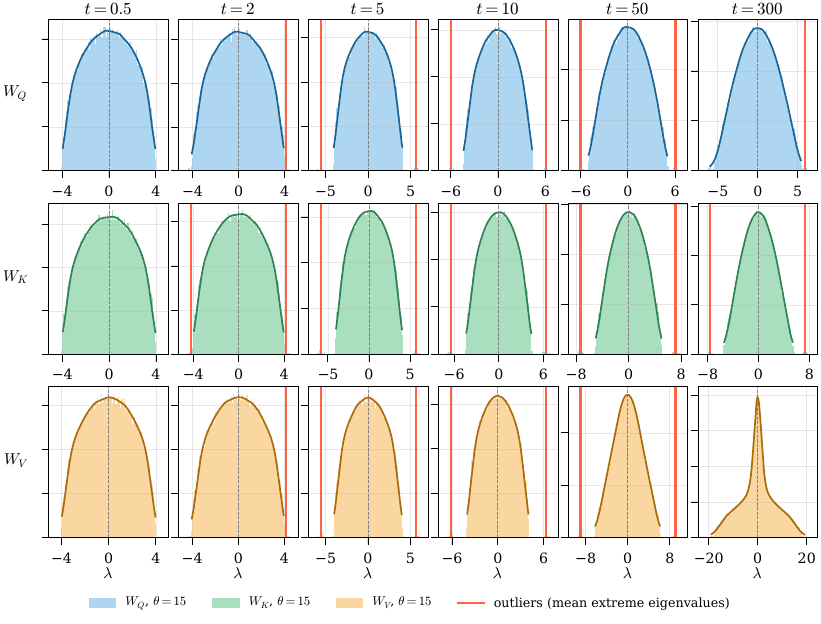}
  \caption{\small One-layer self-attention, $\theta=15$: spectra of $W_Q$, $W_K$, $W_V$ (rows) at six training times; conventions as in Figure~\ref{fig:mlp_bulk}. The outliers of $W_V$ are transient. Those of $W_K$ are orthogonal to the teacher.}
  \label{fig:attention_bulk}
\end{figure}

\paragraph{Why the outliers come in pairs}
For unit vectors $a,r$ with $c=a\cdot r$, the symmetrization of a rank-one matrix $\sigma ar^\top$ satisfies $(ar^\top+ra^\top)(a\pm r)=(c\pm1)(a\pm r)$: it has the two non-zero eigenvalues $\sigma(c+1)$ and $\sigma(c-1)$, with eigenvectors $(a\pm r)/\|a\pm r\|$. In the linear model the learned signal is $\theta\,v\,[(1-e^{-t\Sigma})v]^\top$, whose two directions are close; the negative eigenvalue is small and stays inside the bulk, so a single outlier is visible. A layer inside a product instead learns a component whose hidden side is set by the random initialization of the other layers. Near initialization, gradient flow gives $\dot W_1\supset\theta\,(W_2(0)^\top v)(\Sigma v)^\top$ and $\dot W_2\supset\theta\,v\,(W_1(0)\Sigma v)^\top$. The hidden directions $W_2(0)^\top v$ and $W_1(0)\Sigma v$ are nearly orthogonal to $v$, so $c\approx0$ and $W+W^\top$ shows a pair $\pm\sigma$, each eigenvector carrying about half of the teacher overlap; this is why the overlaps level off near $1/2$. The same holds for self-attention. Writing $\widehat Y=W_VXP$ with $P$ the softmax matrix, gradient flow gives $\dot W_V\supset\theta\,v\,(XPX^\top v)^\top$: $W_V$ writes along $v$ but reads along $XPX^\top v$, which the softmax mixing of the samples makes nearly orthogonal to $v$ (measured overlap $0.03$). $W_Q$ and $W_K$ receive their gradient through the softmax, and their learned components are again non-symmetric: $W_Q$ reads close to $v$ (overlap $0.71$) and writes a random direction, while neither side of the component learned by $W_K$ is aligned with $v$. This explains the pairs in all three rows of Figure~\ref{fig:attention_bulk}, and why the pair of $W_K$ carries no teacher overlap. We checked this mechanism directly: initializing $W_2(0)=I$ makes the learned component of $W_1$ symmetric, and the negative outlier of $W_1$ disappears.

\end{document}